\documentclass[10pt,twocolumn,letterpaper]{article}
\pdfoutput=1

\usepackage{iccv}
\usepackage{times}
\usepackage{epsfig}
\usepackage{graphicx}
\usepackage{amsmath}
\usepackage{amssymb}
\usepackage{amsthm}
\usepackage[all]{xy}
\usepackage[retainorgcmds]{IEEEtrantools}
\usepackage{algorithmic}
\usepackage{algorithm}
\usepackage{subfigure}

\usepackage[pagebackref=true,breaklinks=true,letterpaper=true,colorlinks,bookmarks=false]{hyperref}

\iccvfinalcopy 

\def\iccvPaperID{482} 

\ificcvfinal\fi

\newcommand{\x}{x} 
\newcommand{\y}{y}       
\newcommand{\ny}{{\bar\y}}       
\newcommand{\nynew}{\hat\y}    
\newcommand{\X}{{\cal X}}      
\newcommand{\Y}{{\cal Y}}      
\newcommand{\Sample}{{S}}
\newcommand{\n}{n}       
\newcommand{\WS}{{\cal W}}       
\newcommand{\HS}{{\cal{H}}}       
\newcommand{\w}{\textbf{w}}         
\newcommand{\loss}{\Delta}     

\renewcommand{\vec}[1]{\mathbf{ #1}}

\newcommand{\nslack}{$\n$-slack}
\newcommand{\Nslack}{$\n$-Slack}
\newcommand{\oneslack}{$1$-slack}
\newcommand{\Oneslack}{$1$-Slack}

\newtheorem{optimizationproblem}{Quadratic Program}

\DeclareMathOperator*{\argmax}{argmax}
\DeclareMathOperator*{\argmin}{argmin}

\begin{document}

\long\def\ignore#1{} \long\def\ignoreunlessTR#1{} \def\etal{{\em et
    al.}}  \def\br(#1,#2){{\langle #1,#2 \rangle}} \def\Pr{\mbox{\rm
    Pr}} \def\setof#1{{\left\{#1\right\}}}
\def\suchthat#1#2{\setof{\,#1\mid#2\,}} 
\def\event#1{\setof{#1}} \def\q={\quad=\quad} \def\qq={\qquad=\qquad}
\def\calA{{\cal A}} \def\calB{{\cal B}} \def\calC{{\cal C}}
\def\calD{{\cal D}} \def\calE{{\cal E}} \def\calF{{\cal F}}
\def\calG{{\cal G}} \def\calH{{\cal H}} \def\calI{{\cal I}}
\def\calJ{{\cal J}} \def\calK{{\cal K}} \def\calL{{\cal L}}
\def\calM{{\cal M}} \def\calN{{\cal N}} \def\calO{{\cal O}}
\def\calP{{\cal P}} \def\calQ{{\cal Q}} \def\calR{{\cal R}}
\def\calS{{\cal S}} \def\calT{{\cal T}} \def\calU{{\cal U}}
\def\calV{{\cal V}} \def\calW{{\cal W}} \def\calX{{\cal X}}
\def\calY{{\cal Y}} \def\calZ{{\cal Z}}
\def\psfile[#1]#2{\includegraphics[#1]{#2}}
\def\epsfw#1#2{\includegraphics[width=#1\hsize]{figs/#2}}
\def\assign(#1,#2){\langle#1,#2\rangle} \def\edge(#1,#2){(#1,#2)}

\newtheorem{theorem}{Theorem}[section]
\newtheorem{property}[theorem]{Property}
\newtheorem{observation}[theorem]{Observation}
\newtheorem{corollary}[theorem]{Corollary}
\newtheorem{proposition}[theorem]{Proposition}
\newtheorem{lemma}[theorem]{Lemma}
\newtheorem{definition}[theorem]{Definition} \newtheorem{proofI}{\sc
  Proof:}

\renewcommand{\theproofI}{}
\def\br(#1,#2){{\langle #1,#2 \rangle}} \def\brr(#1,#2,#3){{\langle
    #1,#2,#3 \rangle}} \def\mycaption#1{\caption{#1}\vspace*{0.35in}}

\def\S3SVM{S3SVM}

\title{Structured learning of sum-of-submodular higher order energy functions}

\author{Alexander Fix \and Thorsten Joachims \and Sam Park \and Ramin Zabih}

\maketitle

\begin{abstract}
  Submodular functions can be exactly minimized in polynomial time, and the
  special case that graph cuts solve with max flow \cite{KZ:PAMI04} has had
  significant impact in computer vision
  \cite{BVZ:PAMI01,Kwatra:SIGGRAPH03,Rother:GrabCut04}. In this paper we
  address the important class of sum-of-submodular (SoS) functions
  \cite{Arora:ECCV12,Kolmogorov:DAM12}, which can be efficiently minimized via
  a variant of max flow called submodular flow \cite{Edmonds:ADM77}.  SoS
  functions can naturally express higher order priors involving, e.g., local
  image patches; however, it is difficult to fully exploit their expressive
  power because they have so many parameters.  Rather than trying to formulate
  existing higher order priors as an SoS function, we take a discriminative
  learning approach, effectively searching the space of SoS functions for a
  higher order prior that performs well on our training set.  We adopt a
  structural SVM approach \cite{Joachims/etal/09a,Tsochantaridis/etal/04} and
  formulate the training problem in terms of quadratic programming; as a result we
  can efficiently search the space of SoS priors via an extended cutting-plane
  algorithm.  We also show how the state-of-the-art max flow method for vision
  problems \cite{Goldberg:ESA11} can be modified to efficiently solve the
  submodular flow problem.  Experimental comparisons are made against the
  OpenCV implementation of the GrabCut interactive segmentation technique
  \cite{Rother:GrabCut04}, which uses hand-tuned parameters instead of machine
  learning. On a standard dataset \cite{Gulshan:CVPR10} our method learns
  higher order priors with hundreds of parameter values, and produces
  significantly better segmentations.  While our focus is on binary labeling
  problems, we show that our techniques can be naturally generalized to handle
  more than two labels.
\end{abstract}

\section{Introduction}

Discrete optimization methods such as graph cuts \cite{BVZ:PAMI01,KZ:PAMI04}
have proven to be quite effective for many computer vision problems, including
stereo \cite{BVZ:PAMI01}, interactive segmentation \cite{Rother:GrabCut04} and
texture synthesis \cite{Kwatra:SIGGRAPH03}.  The underlying optimization
problem behind graph cuts is a special case of submodular function
optimization that can be solved exactly using max flow \cite{KZ:PAMI04}.
Graph cut methods, however, are limited by their reliance on first-order
priors involving pairs of pixels, and there is considerable interest in
expressing priors that rely on local image patches such as the popular Field
of Experts model \cite{Roth:IJCV09}.

In this paper we focus on an important generalization of the functions that
graph cuts can minimize, which can express higher-order priors.  These
functions, which \cite{Kolmogorov:DAM12} called Sum-of-Submodular (SoS), can
be efficiently solved with a variant of max flow \cite{Edmonds:ADM77}.  While
SoS functions have more expressive power, they also involve a large number of
parameters.  Rather than addressing the question of which existing higher
order priors can be expressed as an SoS function, we take a discriminative
learning approach and effectively search the space of SoS functions with the
goal of finding a higher order prior that gives strong results on our training
set.\footnote{Since we are taking a discriminative approach, the higher-order
energy function we learn does not have a natural probabilistic interpretation. We are
using the word ``prior'' here somewhat loosely, as is common in computer
vision papers that focus on energy minimization.}

Our main contribution is to introduce the first learning method for training
such SoS functions, and to demonstrate the effectiveness of this approach for
interactive segmentation using learned higher order priors.  Following a
Structural SVM approach
\cite{Joachims/etal/09a,Tsochantaridis/etal/04}, we show that the training
problem can be cast as a quadratic optimization problem over an extended set
of linear constraints. This generalizes large-margin training of pairwise
submodular (a.k.a. regular \cite{KZ:PAMI04}) MRFs
\cite{Anguelov/etal/05,Szummer/etal/08,Taskar/etal/04a}, where submodularity
corresponds to a simple non-negativity constraint.  To solve the training
problem, we show that an extended cutting-plane algorithm can efficiently
search the space of SoS functions.

\subsection{Sum-of-Submodular functions and priors}

A submodular function $f : 2^V \to \mathbb{R}$ on a set $V$ satisfies $ f(S \cap T) + f(S \cup T) \leq f(S) + f(T)$
for all $S,T \subseteq V$.  Such a function is sum-of-submodular (SoS) if we
can write it as
\begin{equation}\label{eqn:sos}
  f(S) = \sum_{C \in \mathcal{C}} f_C(S \cap C) 
\end{equation}
for $\mathcal{C} \subseteq 2^V$ where each $f_C : 2^C \to \mathbb{R}$
is submodular. Research on higher-order priors calls $C \in \mathcal{C}$ a
clique \cite{Ishikawa:TPAMI10}.

Of course, a sum of submodular functions is itself submodular, so we could use
general submodular optimization to minimize an SoS function.  However, general
submodular optimization is $O(n^6)$ \cite{Orlin:MP2009} (which is impractical
for low-level vision problems), whereas we may be able to exploit the
neighborhood structure $\mathcal{C}$ to do better.  For example, if all the
cliques are pairs ($|C| = 2$) the energy function is referred to as regular
and the problem can be reduced to max flow \cite{KZ:PAMI04}.  As mentioned
above, this is the
underlying technique used in the popular graph cuts approach 
\cite{BVZ:PAMI01,Kwatra:SIGGRAPH03,Rother:GrabCut04}.  The key limitation 
is the restriction to pairwise cliques, which does not allow us
to naturally express important higher order priors such as those involving
image patches \cite{Roth:IJCV09}.  The most common approach to
solving higher-order priors with graph cuts, which involves transformation to
pairwise cliques, in practice almost always produces non-submodular functions
that cannot be solved exactly \cite{FGBZ:ICCV11,Ishikawa:TPAMI10}.


\section{Related Work}
\label{sec:related}

Many learning problems in computer vision can be cast as structured
output prediction, which allows learning outputs with spatial
coherence. Among the most popular generic methods for structured
output learning are Conditional Random Fields (CRFs) trained by
maximum conditional likelihood \cite{Lafferty/etal/01}, Maximum-Margin
Markov Networks (M3N) \cite{Taskar/etal/03}, and Structural Support
Vector Machines (SVM-struct)
\cite{Tsochantaridis/etal/04,Joachims/etal/09a}. A key advantage of
M3N and SVM-struct over CRFs is that training does not require
computation of the partition function. Among the two large-margin
approaches M3N and SVM-struct, we follow the SVM-struct methodology
since it allows the use of efficient inference procedures during
training.

In this paper, we will learn submodular discriminant functions. Prior
work on learning submodular functions falls into three categories:
submodular function regression \cite{Balcan/Harvey/11}, maximization
of submodular discriminant functions, and minimization of submodular
discriminant functions.

Learning of submodular discriminant functions where a prediction is computed
through maximization has widespread use in information retrieval, where
submodularity models diversity in the ranking of a search engine
\cite{Yue/Joachims/08a,Lin/Bilmes/12} or in an automatically generated
abstract \cite{Sipos/etal/12a}. While exact (monotone) submodular maximization
is intractible, approximate inference using a simple greedy algorithm has
approximation guarantees and generally excellent performance in practice.

The models considered in this paper use submodular discriminant functions
where a prediction is computed through minimization.  The most popular such
models are regular MRFs \cite{KZ:PAMI04}. Traditionally, the parameters of
these models have been tuned by hand, but several learning methods exist. Most
closely related to the work in this paper are Associative Markov Networks
\cite{Taskar/etal/04a,Anguelov/etal/05}, which take an M3N approach and
exploit the fact that regular MRFs have an integral linear relaxation. These
linear programs (LP) are folded into the M3N quadratic program (QP) that is
then solved as a monolithic QP. In contrast, SVM-struct training using cutting
planes for regular MRFs \cite{Szummer/etal/08} allows graph cut inference also
during training, and \cite{Finley/Joachims/08a,Koppula/etal/11} show that this
approach has interesting approximation properties even the for multi-class
case where graph cut inference is only approximate.  More complex models for
learning spatially coherent priors include separate training for unary and
pairwise potentials \cite{Ladicky/etal/09}, learning MRFs with functional
gradient boosting \cite{Munoz/etal/09}, and the $\calP^n$ Potts models, all of
which have had success on a variety of vision problems.  Note that our general
approach for learning multi-label SoS functions, described in
section~\ref{sec:multi-label}, includes the $\calP^n$ Potts model as a special
case.



\section{SoS minimization}\label{sec:sos}

In this section, we briefly summarize how an SoS function can be
minimized by means of a submodular flow network
(section~\ref{sec:submodular-flow}), and then present our improved
algorithm for solving this minimization (section \ref{sec:ibfs}).

\subsection{SoS minimization via submodular flow}\label{sec:submodular-flow}

Submodular flow is
similar to the max flow problem, in that there is a network of nodes
and arcs on which we want to push flow from $s$ to $t$. However, the
notion of residual capacity will be slightly modified from that of
standard max flow. For a much more complete description see
\cite{Kolmogorov:DAM12}.

We begin with a network $G = (V \cup \{s, t\}, A)$. As in the 
max flow reduction for Graph Cuts, there are source and sink arcs $(s,
i)$ and $(i, t)$ for every $i \in V$. Additionally, for each clique
$C$, there is an arc $(i, j)_C$ for each $i, j \in C$.\footnote{To
  explain the notation, note that $i, j$ might be in multiple cliques
  $C$, so we may have multiple edges $(i, j)$ (that is, $G$ is a
  multigraph). We distinguish between them by the subscript $C$.}

Every arc $a \in A$ also has an associated residual capacity
$\overline{c}_a$. The residual capacity of arcs $(s, i)$ and $(i, t)$
are the familiar residual capacities from max flow: there are
capacities $c_{s,i}$ and $c_{i,t}$ (determined by the unary terms of
$f$), and whenever we push flow on a source or sink arc, we decrease
the residual capacity by the same amount.

For the interior arcs, we need one further piece of information. In
addition to residual capacities, we also keep track of residual clique
functions $\overline{f}_C(S)$, related to the flow values by the
following rule: whenever we push $\delta$ units of flow on arc $(i,
j)_C$, we update $\overline{f}_C(S)$ by
\begin{equation}
\overline{f}_C(S) \leftarrow  \left\{\begin{array}{ll} 
    \overline{f}_C(S) - \delta & i \in S, j\notin S \\
    \overline{f}_C(S) + \delta & i \notin S, j \in S \\
    \overline{f}_C(S) & \mathrm{otherwise} \end{array}\right.
\end{equation}

The residual capacities of the interior arcs are chosen so that the
$\overline{f}_C$ are always nonnegative. Accordingly, we define
$\overline{c}_{i, j, C} = \min_S \{\overline{f}_C(S) \mid i \in S,
  j\notin S\}$.

Given a flow $\phi$, define the set of residual arcs $A_\phi$ as all
arcs $a$ with $\overline{c}_a > 0$. An augmenting path is an $s-t$
path along arcs in $A_\phi$. The following theorem of \cite{Kolmogorov:DAM12}
tells how to optimize $f$ by computing a flow in $G$.

\begin{theorem}\label{thm:augmenting-path} Let $\phi$ be a feasible
  flow such that there is no augmenting path from $s$ to $t$. Let
  $S^\ast$ be the set of all $i \in V$ reachable from $s$ along arcs
  in $A_\phi$. Then $f(S^\ast)$ is the minimum value of $f$ over all
  $S \subseteq V$.
\end{theorem}

\subsection{IBFS for Submodular Flow}
\label{sec:ibfs}

Incremental Breadth First Search (IBFS) \cite{Goldberg:ESA11}, which is
the state of the art in max flow methods for vision applications,
improves the algorithm of \cite{BK:PAMI04} to guarantee polynomial
time complexity.  We now show how to modify IBFS to compute a maximum
submodular flow in $G$.

IBFS is an augmenting paths algorithm: at each step, it finds a path
from $s$ to $t$ with positive residual capacity, and pushes flow along
it. Additionally, each augmenting path found is a shortest $s$-$t$
path in $A_\phi$. To ensure that the paths found are shortest paths,
we keep track of distances $d_s(i)$ and $d_t(i)$ from $s$ to $i$ and
from $i$ to $t$, and search trees $S$ and $T$ containing all nodes at
distance at most $D_s$ from $s$ or $D_t$ from $t$ respectively. Two
invariants are maintained:
\begin{itemize}
\item For every $i$ in $S$, the unique path from $s$ to $i$ in $S$ is
  a shortest $s$-$i$ path in $A_\phi$.
\item For every $i$ in $T$, the unique path from $i$ to $t$ in $T$ is
  a shortest $i$-$t$ path in $A_\phi$.
\end{itemize}

The algorithm proceeds by alternating between forward passes and
reverse passes. In a forward pass, we attempt to grow the source tree
$S$ by one layer (a reverse pass attempts to grow $T$, and is
symmetric).  To grow $S$, we scan through the vertices at distance
$D_s$ away from $s$, and examine each out-arc $(i, j)$ with positive
residual capacity. If $j$ is not in $S$ or $T$, then we add $j$ to $S$
at distance level $D_s + 1$, and with parent $i$. If $j$ is in $T$,
then we found an augmenting path from $s$ to $t$ via the arc $(i, j)$,
so we can push flow on it.

The operation of pushing flow may saturate some arcs (and cause
previously saturated arcs to become unsaturated). If the parent arc of
$i$ in the tree $S$ or $T$ becomes saturated, then $i$ becomes an
orphan. After each augmentation, we perform an adoption step, where
each orphan finds a new parent. The details of the adoption step are
similar to the relabel operation of the Push-Relabel algorithm, in
that we search all potential parent arcs in $A_\phi$ for the neighbor
with the lowest distance label, and make that node our new parent.


In order to apply IBFS to the submodular flow problem, all the basic
datastructures still make sense: we have a graph where the arcs $a$
have residual capacities $\overline{c}_a$, and a maximum flow has been
found if and only if there is no longer any augmenting path from $s$
to $t$.

The main change for the submodular flow problem is that when we increase flow
on an edge $(i, j)_C$, instead of just affecting the residual capacity of that
arc and the reverse arc, we may also change the residual capacities of other arcs
$(i', j')_C$ for $i', j' \in C$. However, the following result 
ensures that this is not a problem.

\begin{lemma} If $(a, b)_C$ was previously saturated, but now
  has residual capacity as a result of increasing flow along $(c, d)$,
  then (1) either $a = d$ or there was an arc $(a, d) \in A_\phi$ and  (2)
  either $b = c$ or there was an arc $(c, b) \in A_\phi$.
\end{lemma}

\begin{corollary}\label{cor:valid-distance}
  Increasing flow on an edge never creates a shortcut between $s$ and
  $i$, or from $i$ to $t$.
\end{corollary}

The proofs are based on results of \cite{Fujishige:96}, and can be found in
the Supplementary Material.



Corollary~\ref{cor:valid-distance} ensures that we never create any
new shorter $s$-$i$ or $i$-$t$ paths not contained in $S$ or $T$. A
push operation may cause some edges to become saturated, but this is
the same problem as in the normal max flow case, and any orpans so
created will be fixed in the adoption step. Therefore, all invariants
of the IBFS algorithm are maintained, even in the submodular flow case.

One final property of the IBFS algorithm involves the use of the
``current arc heuristic'', which is a mechanism for avoiding iterating
through all possible potential parents when performing an adoption
step. In the case of Submodular Flows, it is also the case that
whenever we create new residual arcs we maintain all invariants
related to this current arc heuristic, so the same speedup applies
here. However, as this does not affect the correctness of the
algorithm, and only the runtime, we will defer the proof to the
Supplementary Material.

\textbf{Running time.}
The asymptotic complexity of the standard IBFS algorithm is
$O(n^2m)$. In the submodular-flow case, we still perform the same
number of basic operations. However, note finding residual capacity of
an arc $(i, j)_C$ requires minimizing $\overline{f}_C(S)$ for $S$
separating $i$ and $j$. If $|C| = k$, this can be done in time
$O(k^6)$ using \cite{Orlin:MP2009}. However, for $k << n$, it will
likely be much more efficient to use the $O(2^k)$ naive algorithm of
searching through all values of $\overline{f}_C$. Overall, we add
$O(2^k)$ work at each basic step of IBFS, so if we have $m$ cliques the total
runtime is $O(n^2m2^k)$.

This runtime is better than the augmenting paths algorithm of
\cite{Arora:ECCV12} which takes time $O(nm^22^k)$. Additionally, IBFS
has been shown to be very fast on typical vision inputs, independent
of its asymptotic complexity \cite{Goldberg:ESA11}.

\section{\S3SVM: SoS Structured SVMs}\label{sec:svm}

In this section, we first review the SVM algorithm and its associated
Quadratic Program (section \ref{sec:structured-svm}). We then decribe a
general class of SoS discriminant functions which can be learned by SVM-struct
(section \ref{sec:submodular-feature}) and explain this learning procedure
(section \ref{sec:solving-qp}).  Finally, we generalize SoS functions to the
multi-label case (section \ref{sec:multi-label}).

\subsection{Structured SVMs}\label{sec:structured-svm}

Structured output prediction describes the problem of learning a
function $h: \X \longrightarrow \Y$
where $\X$ is the space of inputs, and $\Y$ is the space of
(multivariate and structured) outputs for a given problem.
To learn $h$, we assume that a training sample of input-output pairs
$\Sample=((\x_1,\y_1),\dots,(\x_n,\y_n)) \in (\X \times \Y)^n$
is available and drawn i.i.d. from an unknown distribution. The goal
is to find a function $h$ from some hypothesis space $\HS$ that has
low prediction error, relative to a loss function $\loss(\y,\ny)$. The
function $\loss$ quantifies the error associated with predicting $\ny$
when $\y$ is the correct output value. For example, for image
segmentation, a natural loss function might be the Hamming distance
between the true segmentation and the predicted labeling.

The mechanism by which Structural SVMs finds a hypothesis $h$ is to
learn a discriminant function $f:\X \times \Y \to \mathbb{R}$ over
input/output pairs. One derives a prediction for a given input $\x$ by
minimizing $f$ over all $\y \in \Y$.\footnote{Note that the use of
  minimization departs from the usual language of
  \cite{Tsochantaridis/etal/04,Joachims/etal/09a} where the hypothesis
  is $\argmax f_w(\x, \y)$. However, because of the prevalence of cost
  functions throughout computer vision, we have replaced $f$ by $-f$
  throughout.} We will write this as
$h_\w(\x) = \argmin_{\y \in \Y} f_\w(\x,\y)$.
We assume  $f_\w(\x,\y)$ is linear in two quantities $\w$ and $\Psi$
$f_\w(\x,\y)=\w^T \Psi(\x,\y)$
where $\w \in \mathbb{R}^N$ is a parameter vector and $\Psi(\x,\y)$ is a
feature vector relating input $\x$ and output $\y$. Intuitively, one
can think of $f_\w(\x,\y)$ as a cost function that measures
how poorly the output $\y$ matches the given input $\x$.


Ideally, we would find weights $\w$ such that the hypothesis $h_\w$
always gives correct results on the training set. Stated another way,
for each example $\x_i$, the correct prediction $\y_i$ should have low
discriminant value, while incorrect predictions $\ny_i$ with large
loss should have high discriminant values. We write this
constraint as a linear inequality in $\w$
\begin{equation}\label{eqn:constr-nslack}
\w^T\Psi(\x_i, \ny_i) \geq \w^T\Psi(\x_i, \y_i) + \Delta(\y_i, \ny_i)
: \forall \ny \in \Y.
\end{equation}
It is convenient to define $\delta\Psi_i(\ny) = \Psi(\x_i, \ny) -
\Psi(\x_i, \y_i)$, so that the above inequality becomes
$\w^T\delta\Psi_i(\ny_i) \geq \Delta(\y_i,\ny_i)$.

Since it may not be possible to satisfy all these conditions exactly,
we also add a slack variable to the constraints for example
$i$. Intuitively, slack variable $\xi_i$ represents the maximum
misprediction loss on the $i$th example. Since we want to minimize the
prediction error, we add an objective function which penalizes large
slack. Finally, we also penalize $\|w\|^2$ to discourage overfitting,
with a regularization parameter $C$ to trade off these costs.

\noindent
\begin{optimizationproblem}{\sc \Nslack{} Structural
    SVM} \label{op:nslack_mr_primal}
  \vspace{-.1in}  
  \begin{eqnarray*}
      \min_{\w,\vec{\xi}\ge \vec{0}}& & \frac{1}{2} \w^T\w +
      \frac{C}{n}\sum_{i=1}^{n}\xi_i \label{eq:primalmargin} \\
      \forall i, \forall \ny_i \in \Y: & &  \w^T \delta\Psi_i(\ny_i) \geq
      \loss(\y_i,\ny_i) - \xi_i \nonumber
    \end{eqnarray*}
    \vspace*{-0.6cm}
\end{optimizationproblem}

\subsection{Submodular Feature Encoding}\label{sec:submodular-feature}

We now apply the Structured SVM (SVM-struct) framework to the problem of
learning SoS functions. 

For the moment, assume our prediction task is to assign a binary
label for each element of a base set $V$.
We will cover the multi-label case in section
\ref{sec:multi-label}. Since the labels are binary, prediction
consists of assigning a subset $S \subseteq V$ for each input (namely
the set $S$ of pixels labeled 1).

Our goal is to construct a feature vector $\Psi$ that, when used with
the SVM-struct algorithm of section \ref{sec:structured-svm}, will allow us
to learn sum-of-submodular energy functions. Let's begin with the
simplest case of learning a discriminant function $f_{C,\w}(S) =
\w^T\Psi(S)$, defined only on a single clique and which does not
depend on the input $x$.

Intuitively, our parameters $\w$ will correspond to the table of
values of the clique function $f_C$, and our feature vector $\Psi$
will be chosen so that $\w_{S} = f_C(S)$. We can accomplish this by
letting $\Psi$ and $\w$ have $2^{|C|}$ entries, indexed by subsets $T
\subseteq C$, and defining $\Psi_{T}(S) = \delta_{T}(S)$ (where
$\delta_T(S)$ is 1 if $S = T$ and 0 otherwise). Note that, as we
claimed,
\begin{equation}
f_{C,\w}(S) = \w^T\Psi(S) = \sum_{T \subseteq C} \w_{T} \delta_T(S) =
\w_{S}.
\end{equation}
\vspace{-.15in}

If our parameters $\w_{T}$ are allowed to vary over all
$\mathbb{R}^{2^{|C|}}$, then $f_C(S)$ may be an arbitrary function
$2^{C} \to \mathbb{R}$, and not necessarily submodular. However, we
can enforce submodularity by adding a number of linear
inequalities. Recall that $f$ is submodular if and only if $f(A\cup B)
+ f(A \cap B) \leq f(A) + f(B)$. Therefore, $f_{C,\w}$ is submodular
if and only if the parameters satisfy
\begin{equation}\label{eqn:submodular-constraints}
\w_{A\cup B} + \w_{ A \cap B} \leq \w_{A} + \w_{B} : \forall A,B
\subseteq C
\end{equation} 

These are just linear constraints in $\w$, so we can add them as
additional constraints to Quadratic Program
\ref{op:nslack_mr_primal}. There are $O(2^{|C|})$ of them, but each
clique has $2^{|C|}$ parameters, so this does not increase the
asymptotic size of the QP.

\begin{theorem} By choosing feature vector $\Psi_{T}(S) = \delta_{T}(S)$
  and adding the linear constraints (\ref{eqn:submodular-constraints})
  to Quadratic Program~\ref{op:nslack_mr_primal}, the learned
  discriminant function $f_\w(S)$ is the maximum margin function $f_C$,
  where $f_C$ is allowed to vary over all possible submodular functions
  $f : 2^C \to \mathbb{R}$.
\end{theorem}
\begin{proof} By adding constraints (\ref{eqn:submodular-constraints})
  to the QP, we ensure that the optimal solution $\w$ is defines a
  submodular  $f_\w$. Conversely, for any
  submodular function $f_C$, there is a feasible $\w$ defined by $\w_T
  = f_C(T)$, so the optimal solution to the QP must be the
  maximum-margin such function.
\end{proof}
To introduce a dependence on the data $\x$, we can define
$\Psi^\textrm{data}$ to be $\Psi_T^\textrm{data}(S,\x) = \delta_T(S) \Phi(\x)$ for an arbitrary
nonnegative function $\Phi : \X \to \mathbb{R}_{\geq 0}$.

\begin{corollary}\label{cor:data-dependence} With feature vector
  $\Psi^\mathrm{data}$ and adding linear constraints
  (\ref{eqn:submodular-constraints}) to
  QP~\ref{op:nslack_mr_primal}, the learned discriminant function is
  the maximum margin function $f_C(S)\Phi(x)$, where $f_C$ is allowed
  to vary over all possible submodular functions.
\end{corollary}
\begin{proof} Because $\Phi(x)$ is nonnegative, constraints
  (\ref{eqn:submodular-constraints}) ensure that the discriminant
  function is again submodular.
\end{proof}

Finally, we can learn multiple clique potentials simultaneously. If we
have a neighborhood structure $\mathcal{C}$ with $m$ cliques, each
with a data-dependence $\Phi_C(\x)$, we create a feature vector
$\Psi^\textrm{sos}$ composed of concatenating the $m$ different
features $\Psi^\mathrm{data}_C$.

\begin{corollary} With feature vector $\Psi^\mathrm{sos}$, and adding
  a copy of the constraints (\ref{eqn:submodular-constraints}) for
  each clique $C$, the learned $f_\w$ is the maximum margin $f$ of the
  form
\begin{equation}
f(\x, S) = \sum_{C \in \mathcal{C}} f_C(S)\Phi_C(\x)
\end{equation}
where the $f_C$ can vary over all possible submodular functions on the
cliques $C$.
\end{corollary}

\subsection{Solving the quadratic program}\label{sec:solving-qp}

\begin{algorithm}[t]
\begin{algorithmic}[1]
\STATE Input: $\Sample=((\x_1,\y_1),\ldots,(\x_n,\y_\n))$, $C$, $\epsilon$
\STATE $\WS \leftarrow \emptyset$
\REPEAT
  \STATE Recompute the QP solution with the current constraint set:
  \newline
  $(\w, \vec{\xi}) \leftarrow \argmin_{\w,\vec{\xi}\geq 0}
      \frac{1}{2} \w^T\w + C \vec{\xi} 
      \newline
      \hspace*{0.5cm}s.t.\ \textrm{for all } (\ny_1, \ldots, \ny_n) \in \WS: 
      \newline
      \hspace*{1.0cm}\frac{1}{n} \w^T\sum_{i = 1}^n \delta\Psi_i(\ny_i)\geq \frac{1}{n} \sum_{i = 1}^n
      \Delta(\y_i, \ny_i) - \vec{\xi} 
      \newline
      \hspace*{0.5cm}s.t.\ \textrm{for all } C \in \mathcal{C},\ A,B\subseteq C :
      \newline
      \hspace*{1.0cm}\w_{C,A\cup B} + \w_{C,A \cap B} \leq \w_{C,A} + \w_{C,B}$
  \FOR{i=1,...,\n} 
    \STATE Compute the maximum violated constraint:
    \newline 
    $\nynew_i \leftarrow \argmin_{\nynew \in \Y} \{\w^T
    \Psi(\x_i,\nynew) - \loss(\y_i,\nynew) \}$\newline
    by using IBFS to minimize $f_\w(\x_i, \nynew) - \Delta(\y_i, \nynew)$.
   \ENDFOR
   \STATE $\WS \leftarrow \WS \cup\{(\nynew_1, \ldots, \nynew_n)\}$
 \UNTIL {the slack of the max-violated constraint is $\leq \xi + \epsilon$.}\hspace{-2ex}
 \RETURN($\w$,$\xi$)
\end{algorithmic}
\caption{\label{alg:nslack_mr_alg}: S3SVM via the \Oneslack{} Formulation.}
\end{algorithm}

The \nslack{} formulation for SSVMs (QP~\ref{op:nslack_mr_primal})
makes intuitive sense, from the point of view of minimizing the
misprediction error on the training set. However, in practice it is
better to use the \oneslack{} reformulation of this QP from
\cite{Joachims/etal/09a}. Compared to \nslack{}, the \oneslack{} QP
can be solved several orders of magnitude faster in practice, as well
as having asymptotically better complexity.

The \oneslack{} formulation is an equivalent QP which replaces the $n$
slack variables $\xi_i$ with a single variable $\xi$. The loss constraints
(\ref{eqn:constr-nslack}) are replaced with constraints penalizing
the sum of losses across all training examples. We also include
submodular constraints on $\w$.

\begin{optimizationproblem}{\sc 1-Slack Structural SVM} \label{op:oneslack_mr_primal}
\begin{eqnarray*}
\min_{\w,\xi \ge 0} & & \frac{1}{2} \w^T \w+ C \: \xi \qquad 
\text{s.t. } \: \forall (\ny_1, ..., \ny_\n) \in \Y^\n : \\
& & \hspace{-.9cm}\frac{1}{\n} \w^T \sum_{i=1}^{\n}
\delta\Psi_i(\ny_i) \ge \frac{1}{\n}\sum_{i=1}^{\n}\loss(\y_i,\ny_i) -
\xi \\ 
& & \hspace{-.9cm}\forall C \in \mathcal{C},\ A,B\subseteq C: \: \w_{C,A\cup B} +
\w_{C, A\cap B} \leq \w_{C,A} + \w_{C,B}
\end{eqnarray*}
\vspace*{-0.8cm}
\end{optimizationproblem}

Note that we have a constraint for each tuple $(\ny_1, \ldots, \ny_\n)
\in \Y^\n$, which is an exponential
sized set. Despite the large set of constraints, we can solve this
QP to any desired precision $\epsilon$ by using the
cutting plane algorithm. This algorithm keeps track of a set $\WS$ of
current constraints, solves the current QP with regard to those
constraints, and then given a solution $(\w, \xi)$, finds the most
violated constraint and adds it to $\WS$. Finding the most violated
constraint consists of solving for each example $\x_i$ the problem
\begin{equation}\label{eqn:inference}
  \nynew_i = \argmin_{\nynew \in \Y} f_\w(\x, \nynew) -
  \Delta(\y_i, \nynew).
\end{equation}
Since the features $\Psi$ ensure that $f_\w$ is SoS, then as long as
$\Delta$ factors as a sum over the cliques $\mathcal{C}$ (for
instance, the Hamming loss is such a function), then
(\ref{eqn:inference}) can be solved with Submodular IBFS. Note that this
also allows us to add arbitrary additional features for learning the unary
potentials as well. Pseudocode for the entire \S3SVM learning is given
in Algorithm~\ref{alg:nslack_mr_alg}.

\subsection{Generalization to multi-label prediction}\label{sec:multi-label}

Submodular functions are intrinsically binary functions.
In order to handle the multi-label case, we use expansion moves
\cite{BVZ:PAMI01} to reduce the multi-label optimization problem to a
series of binary subproblems, where each pixel may either switch to a
given label $\alpha$ or keep its current label.  If every binary
subproblem of computing the optimal expansion move is an SoS problem,
we will call the original multi-label energy function an SoS expansion
energy.

Let $L$ be our label set, with output space $\Y = L^V$. Our learned
function will have the form
$f(\y) = \sum_{C\in \mathcal{C}} f_C(\y_C)$
where $f_C : L^C \to
\mathbb{R}$.  For a clique $C$ and label $\ell$, define $C_\ell = \{i
\mid y_i = \ell\}$, i.e., the subset of $C$ taking label $\ell$.

\begin{theorem}\label{lemma:sos-exp} If all the clique functions are of the form 
\begin{equation}
f_C(\y_C) = \sum_{\ell \in L} g_\ell(C_\ell)
\end{equation}
where each $g_\ell$ is submodular, then any expansion move for the
multi-label energy function $f$ will be SoS.
\end{theorem}
\begin{proof}
  Fix a current labeling $\y$, and let $B(S)$ be the energy when the
  set $S$ switches to label $\alpha$. We can write $B(S)$ in terms of
  the clique functions and sets $C_\ell$ as
\begin{equation}
B(S) = \sum_{C \in \mathcal{C}} \bigg(g_\alpha(C_\alpha \cup S) +
  \sum_{\ell\neq\alpha} g_\ell(C_\ell \setminus S)\bigg)
\end{equation}

We use a fact from the theory of submodular functions: if $f(S)$ is
submodular, then for any fixed $T$ both $f(T \cup S)$ and $f(T
\setminus S)$ are also submodular. Therefore, $B(S)$ is SoS.
\end{proof}

Theorem~\ref{lemma:sos-exp} characterizes a large class of SoS
expansion energies. These functions generalize commonly used
multi-label clique functions, including the $\calP^n$ Potts model
\cite{Torr:TPAMI09}.  The $\calP^n$ model pays cost $\lambda_i$ when
all pixels are equal to label $i$, and $\lambda_{\max}$ otherwise.  We
can write this as an SoS expansion energy by letting $g_\ell(S) =
\lambda_i - \lambda_{\max}$ if $S = C$ and otherwise $0$.  Then,
$\sum_\ell g_\ell(S)$ is equal to the $\calP^n$ Potts model, up to an
additive constant.  Generalizations such as the robust $\calP^n$ model
\cite{Kohli:IJCV09} can be encoded in a similar fashion.  Finally, in
order to learn these functions, we let $\Psi$ be composed of copies of
$\Psi^\mathrm{data}$ --- one for each $g_\ell$, and add
corresponding copies of the constraints
(\ref{eqn:submodular-constraints}). 

As a final note: even though the individual expansion moves can be
computed optimally, $\alpha$-expansion still may not find the global
optimum for the multi-labeled energy. However, in practice
$\alpha$-expansion finds good local optima, and has been used for
inference in Structural SVM with good results, as in
\cite{Koppula/etal/11}. 

\section{Experimental Results}

\def\subfigw#1#2{{\includegraphics[width=#1]{Images/#2}}}
\def\subfigh#1#2{{\includegraphics[height=#1]{Images/#2}}}

\def\wid{0.24\linewidth}
\begin{figure}
  \begin{center}
      \subfigw{\wid}{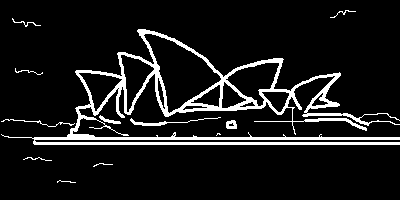}
      \subfigw{\wid}{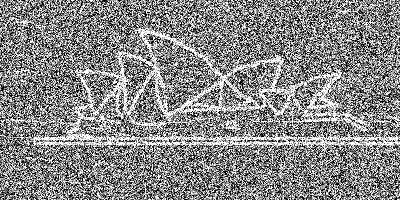}
      \subfigw{\wid}{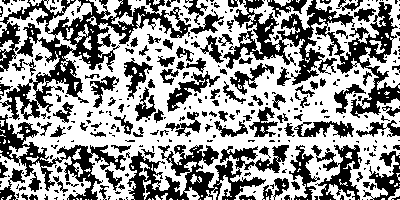}
      \subfigw{\wid}{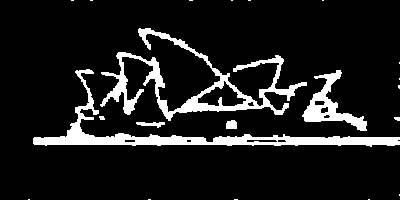}
      \vspace{-2ex}\\
      \subfigw{\wid}{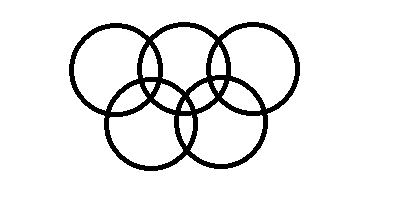}
      \subfigw{\wid}{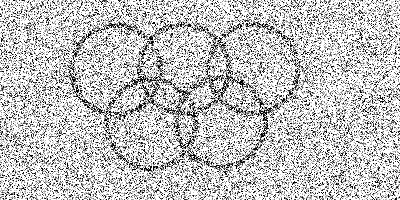}
      \subfigw{\wid}{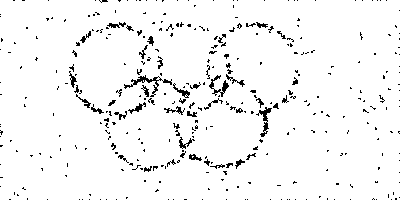}
      \subfigw{\wid}{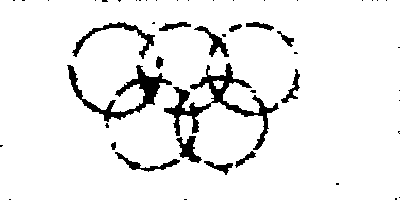}
      \vspace{-2ex}\\
      \subfigw{\wid}{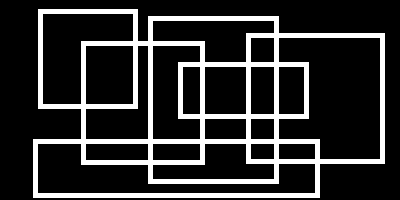}
      \subfigw{\wid}{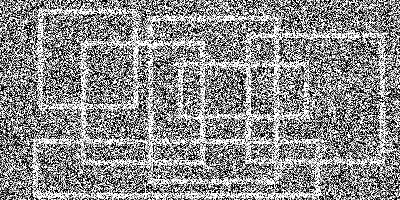}
      \subfigw{\wid}{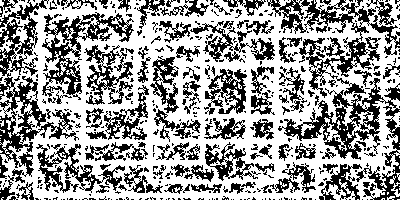}
      \subfigw{\wid}{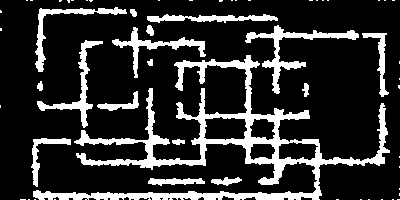}
      \vspace{0ex} 
    \caption{Example images from the binary segmentation results. From
    left to right, the columns are (a) the original image (b) the
    noisy input (c) results from Generic Cuts \cite{Arora:ECCV12} (d) our results.}
    \label{fig:binary-images}
  \end{center}
\vspace{-.3in}
\end{figure}

In order to evaluate our algorithms, we focused on binary denoising and
interactive segmentation.  For binary denoising, Generic Cuts
\cite{Arora:ECCV12} provides the most natural comparison since it is a
state-of-the-art method that uses SoS priors.  For interactive segmentation
the natural comparison is against GrabCut \cite{Rother:GrabCut04}, where we
used the OpenCV implementation. We ran our general \S3SVM{} method, which can
learn an arbitrary SoS function, an also considered the
special case of only using pairwise priors.  
For both the denoising and segmentation applications, we significantly improve
on the accuracy of the hand-tuned energy functions.

\subsection{Binary denoising}

Our binary denoising dataset consists of a set of 20 black and white
images. Each image is $400\times 200$ and either a set of geometric lines, or
a hand-drawn sketch (see Figure~\ref{fig:binary-images}).  We were unable to
obtain the original data used by \cite{Arora:ECCV12}, so we created our own
similar data by adding independent Gaussian noise at each pixel.

For denoising, the hand-tuned Generic Cuts algorithm of \cite{Arora:ECCV12} posed
a simple MRF, with unary pixels equal to the absolute valued distance from the
noisy input, and an SoS prior, where each $2\times 2$ clique
penalizes the square-root of the number of edges with different labeled
endpoints within that clique. There is a single parameter $\lambda$, which is
the tradeoff between the unary energy and the smoothness term.  The
neighborhood structure $\mathcal{C}$ consists of all $2\times 2$ patches of
the image.

Our learned prior includes the same unary terms and clique structure,
but instead of the square-root smoothness prior, we learn a clique
function $g$ to get an MRF
$E_\textrm{SVM}(y) = \sum_i |y_i - x_i| + \sum_{C \in \mathcal{C}}
g(y_C)$.
Note that each clique has the same energy as every other, so this is
analogous to a graph cuts prior where each pairwise edge has the same
attractive potential. Our energy function has 16 total parameters (one for each
possible value of $g$, which is defined on $2\times 2$ patches).

We randomly divided the 20 input images into 10 training images and 10 test
images. The loss function was the Hamming distance between the correct,
un-noisy image and the predicted image. To hand tune the value $\lambda$, we
picked the value which gave the minimum pixel-wise error on the training set.
\S3SVM{} training took only 16 minutes.

Numerically, \S3SVM{} performed signficantly better than the hand-tuned method,
with an average pixel-wise error of only 4.9\% on the training set, compared
to 28.6\% for Generic Cuts. The time needed to do inference after training was
similar for both methods: 0.82 sec/image for \S3SVM{} vs. 0.76 sec/image for
Generic Cuts. Visually, the \S3SVM{} images are significantly cleaner looking,
as shown in Figure~\ref{fig:binary-images}.

\subsection{Interactive segmentation}

\def\txtw{.9in}
\def\txtraise{.45in}
\def\hgt{1.0in}
\def\negspace{-.12in}
\begin{figure*}[t]
  \begin{center}
      \subfigure{\makebox[\txtw][l]{\raisebox{\txtraise}{Input}}}
      \subfigh{\hgt}{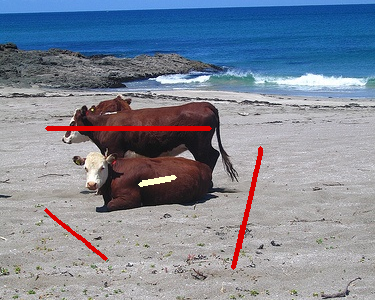}
      \subfigh{\hgt}{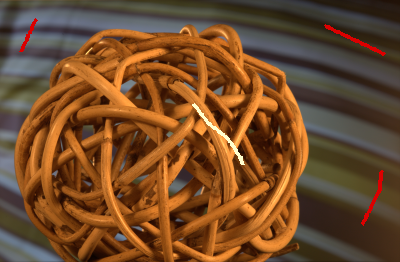}
      \subfigh{\hgt}{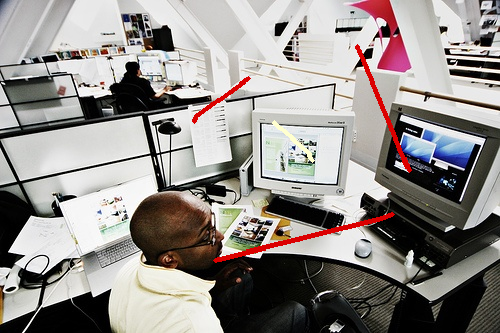}
      \subfigh{\hgt}{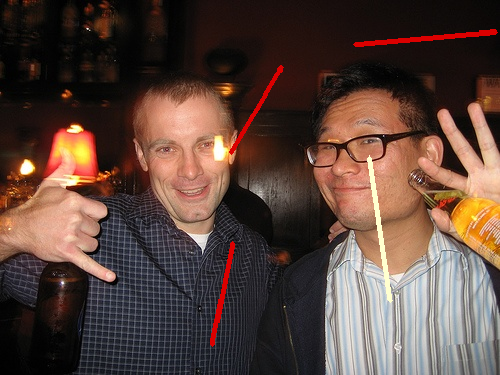}
      \vspace{\negspace}
      \\

      \subfigure{\makebox[\txtw][l]{\raisebox{\txtraise}{GrabCut}}}
     \subfigh{\hgt}{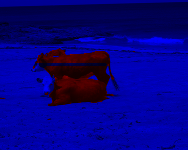}
      \subfigh{\hgt}{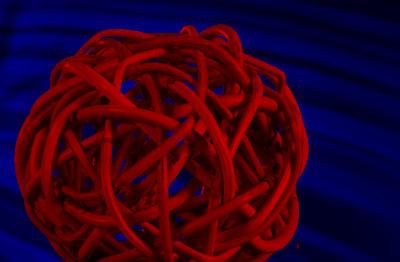}
      \subfigh{\hgt}{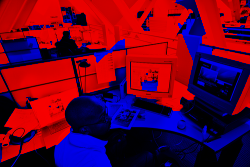}
      \subfigh{\hgt}{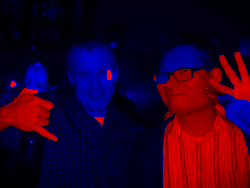}
      \vspace{\negspace}
      \\
      \subfigure{\makebox[\txtw][l]{\raisebox{\txtraise}{S3SVM-AMN}}}
      \subfigh{\hgt}{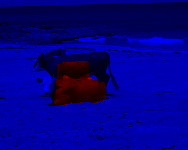}
      \subfigh{\hgt}{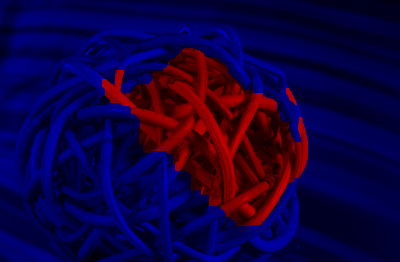}
      \subfigh{\hgt}{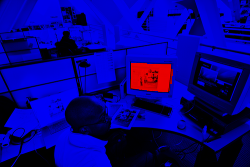}
      \subfigh{\hgt}{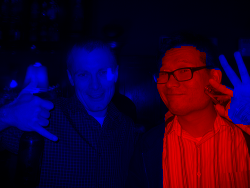}
      \vspace{\negspace}
      \\

      \subfigure{\makebox[\txtw][l]{\raisebox{\txtraise}{S3SVM}}}
     \subfigh{\hgt}{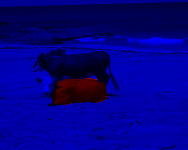}
      \subfigh{\hgt}{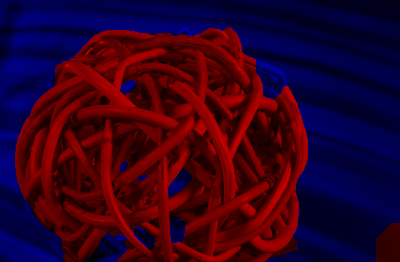}
      \subfigh{\hgt}{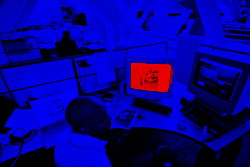}
      \subfigh{\hgt}{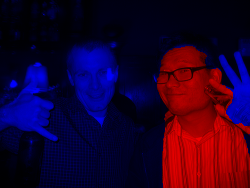}
      \vspace{0in} 
     
      \caption{Example images from binary segmentation results. Input with
        user annotations are shown at top, with results below.}
    \label{fig:interactive-images}
  \end{center}
\vspace{-.3in}
\end{figure*}

The input to interactive segmentation is a color image, together with
a set of sparse foreground/background annotations provided by the
user. See Figure \ref{fig:interactive-images} for examples. From the
small set of labeled foreground and background pixels, the prediction
task is to recover the ground-truth segmentation for the whole
image.

Our baseline comparison is the Grabcut algorithm, which solves a
pairwise CRF. The unary terms of the CRF are obtained by fitting a
Gaussian Mixture Model to the histograms of pixels labeled as being
definitely foreground or background. The pairwise terms are a standard
contrast-sensitive Potts potential, where the cost of pixels $i$ and
$j$ taking different labels is equal to $\lambda\cdot
\textrm{exp}(-\beta |x_i - x_j|)$ for some hand-coded parameters
$\beta, \lambda$.  Our primary comparison is against the OpenCV implementation of Grabcut,
available at \url{www.opencv.org}. 

As a special case, our algorithm can be applied to pairwise-submodular energy functions,
for which it solves the same optimization problem as in Associative Markov
Networks (AMN's) \cite{Taskar/etal/04a,Anguelov/etal/05}.  Automatically
learning parameters allows us to add a large number of learned unary features
to the CRF. 

As a result, in addition to the smoothness parameter $\lambda$, we also
learn the relative weights of approximately 400 features
describing the color values near a pixel, and relative distances to the
nearest labeled foreground/background pixel. Further details on these features
can be found in the Supplementary Material.  We refer to this method as
\S3SVM-AMN.

Our general \S3SVM{} method can incorporate higher-order priors instead of
just pairwise ones. In addition to the unary features used in \S3SVM-AMN, we
add a sum-of-submodular higher-order CRF. Each $2\times 2$ patch in the image
has a learned submodular clique function. To obtain the benefits of the
contrast-sensitive pairwise potentials for the higher-order case, we cluster
(using $k$-means) the $x$ and $y$ gradient responses of each patch into 50
clusters, and learn one submodular potential for each cluster.  Note that
\S3SVM{} automatically allows learning the entire energy function, including
the clique potentials and unary potentials (which come from the data)
simultaneously.

We use a standard interactive segmentation dataset from
\cite{Gulshan:CVPR10} of 151 images with annotations, together with
pixel-level segmentations provided as ground truth. These images were
randomly sorted into training, validation and testing sets, of size
75, 38 and 38 respectively. We trained both \S3SVM{}-AMN and \S3SVM{}
on the training set for various values of the regularization parameter
$c$, and picked the value $c$ which gave the best accuracy on the
validation set, and report the results of that value $c$ on the test
set.

The overall performance is shown in the table below. Training time is
measured in seconds, and testing time in seconds per image.  Our
implementation, which used the submodular flow algorithm based on IBFS
discussed in section~\ref{sec:ibfs}, will be made freely available
under the MIT license.

\begin{flushleft}
\begin{tabular}{|lrrr|}
\hline
{\bf Algorithm}    & {\bf Average error}   & {\bf Training} & {\bf Testing} \\ 
\hline
Grabcut      & 10.6$\pm$ 1.4\%   & n/a                 & 1.44 \\
\S3SVM-AMN & 7.5$\pm$ 0.5\%    & 29000               & 0.99 \\
\S3SVM      & 7.3$\pm$ 0.5\%    & 92000               & 1.67\\
\hline
\end{tabular}
\end{flushleft}

Learning and validation was performed 5 times with independently sampled
training sets. The averages and standard deviations shown above are from these
5 samples.

\def\wid{0.485\linewidth}
\begin{figure}[t]
  \begin{center}
      \subfigw{\wid}{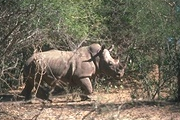} \hspace*{.01in}
      \subfigw{\wid}{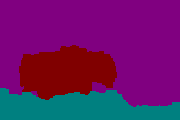}\\
      \vspace{0ex} 
    \caption{A multi-label segmentation result, on data from
      \cite{Zemel:CVPR04}. The purple label represents vegetation, red is
      rhino/hippo and blue is ground.  There are 7 labels in the input
      problem, though only 3 are present in the output we obtain on this particular
      image.}
    \label{fig:multilabel}
  \end{center}
\vspace{-.3in}
\end{figure}

While our focus is on binary labeling problems, we have conducted some preliminary
experiments with the multi-label version of our method described in
section~\ref{sec:multi-label}. 
A sample result is shown in figure~\ref{fig:multilabel}, using an image taken
the Corel dataset used in \cite{Zemel:CVPR04}.  

{\small 
\bibliographystyle{ieee} 

}
\end{document}